\documentclass[twoside,11pt]{article}
\pdfoutput=1
\usepackage{geometry}
\usepackage[colorlinks=true]{hyperref}
\usepackage{alex}
\usepackage{algorithm,algorithmic}

\usepackage{tikz}
\usepackage{pgfplots}
\usepackage{bchart}
\usepackage{color}
\usepackage{xspace}
\usepackage[tight]{subfigure}

\usetikzlibrary{calc,positioning,arrows}
\usetikzlibrary{automata,positioning}
\usetikzlibrary{fit}         
\usetikzlibrary{backgrounds} 
\usetikzlibrary{positioning}
\usetikzlibrary{shadows}
\usepgflibrary{shapes}

\tikzstyle{observed}=[circle, thick, minimum size=0.9cm, draw=black!100, fill=black!20]
\tikzstyle{latent}=[circle, thick, minimum size=0.9cm, draw=black!80]
\tikzstyle{plate}=[rectangle, thick, inner sep=0.25cm, draw=black!100]
\tikzstyle{shadeplate}=[rectangle, thick, inner sep=0.4cm, draw=black!100]
\tikzstyle{table}=[circle,fill=blue!20,draw=black!100,inner sep=1pt, minimum size=30pt]
\tikzstyle{client}=[rectangle,fill=blue!20,draw=black!100,inner sep=1pt, minimum size=12pt]

\newcommand{\method}{ACCAMS\xspace}
\newcommand{\bmethod}{bACCAMS\xspace}
\newcommand{\myparagraph}[1]{\vspace{1mm}\noindent{\bfseries{#1}.}}

\newenvironment{proof}{\par\noindent{\bf Proof\ }}{\hfill\BlackBox\\[2mm]}

\DeclareGraphicsExtensions{.pdf,.eps,.png,.jpg}

\begin{document}

\title{ACCAMS \\ Additive Co-Clustering to Approximate Matrices Succinctly} 

\author{
  Alex Beutel\\
  Computer Science, CMU, Pittsburgh, PA, {\tt abeutel@cs.cmu.edu}
  \and
  Amr Ahmed\\
  Google~Strategic Technologies, Mountain View, CA, {\tt amra@google.com}
  \and 
  Alexander J. Smola \\
  Machine Learning, CMU, Pittsburgh, PA, {\tt alex@smola.org} \\
  Google~Strategic Technologies, Mountain View CA, USA
}

\maketitle

\begin{abstract}
  Matrix completion and approximation are popular tools to capture a
  user's preferences for recommendation and to approximate missing
  data.  Instead of using low-rank factorization we take a drastically
  different approach, based on the
  simple insight that an additive model of co-clusterings allows
  one to approximate matrices efficiently.  This allows us to build a
  concise model that, per bit of model 
  learned, significantly beats all factorization approaches to
  matrix approximation.  Even more surprisingly, we find that 
  summing over small co-clusterings is more effective in
  modeling matrices than classic co-clustering, which uses just one
  large partitioning of the matrix. 

  Following Occam's razor principle suggests that the simple structure
  induced by our model better captures the latent preferences and
  decision making processes present in the real world than classic
  co-clustering or matrix factorization. We provide an iterative
  minimization algorithm, a collapsed Gibbs sampler, theoretical
  guarantees for matrix approximation, and excellent empirical
  evidence for the efficacy of our approach.  
  We achieve state-of-the-art results on the Netflix
  problem with a fraction of the model complexity.
\end{abstract}

\section{Introduction}
\label{sec:intro}

\begin{figure}[t]
\centering
\includegraphics[width=0.6\columnwidth]{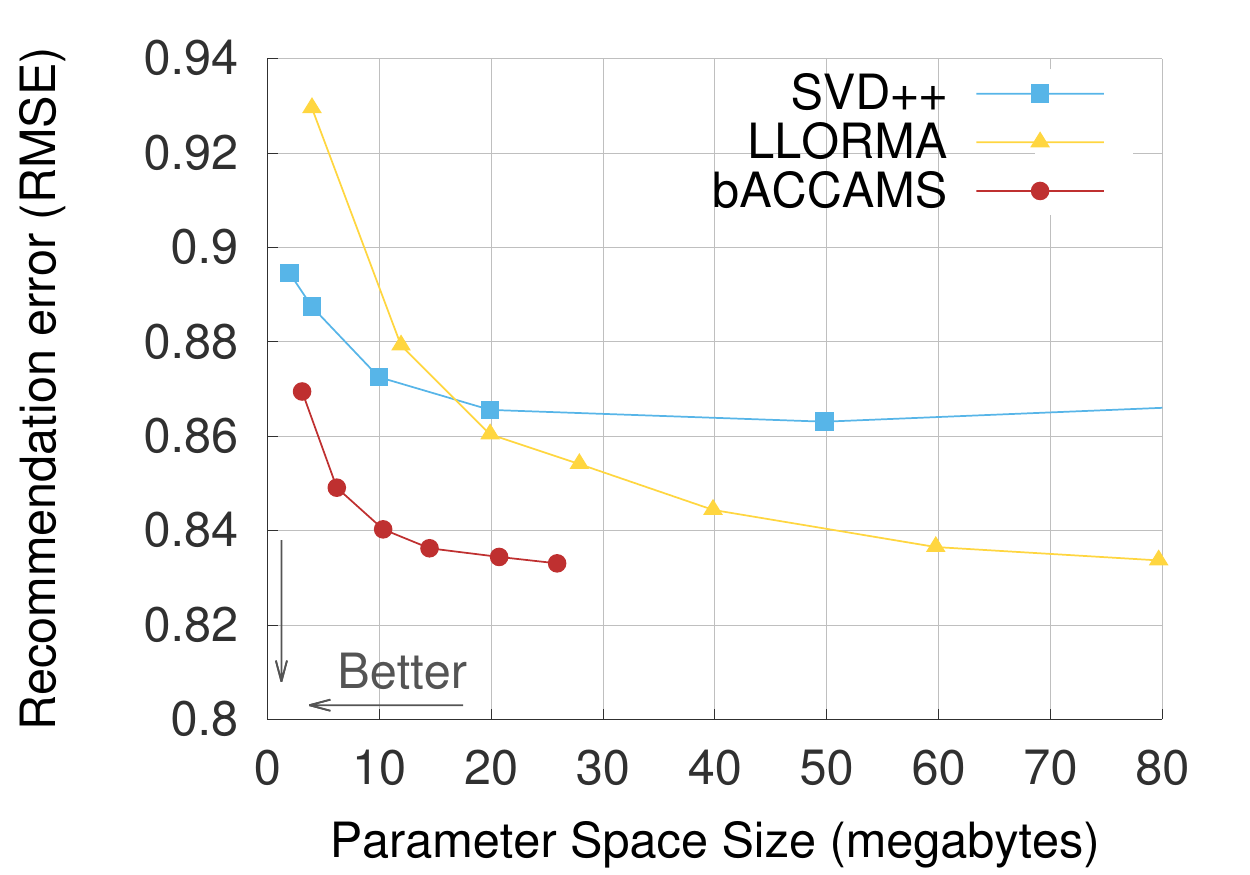} 
\caption{Accuracy of \method on Netflix, compared to
\cite{KorBelVol09} and \cite{lee2013local}. Note that our model achieves state
of the art accuracy at a fraction of the model size.
\label{fig:netflix_test_crownjewel}}
\end{figure}

Given users' ratings of movies or products, how can we model a user's
preferences for different types of items and recommend other items that the
user will like?
This problem, often referred to as the Netflix problem, has generated a flurry
of research in collaborative filtering, with a variety of proposed matrix
factorization models and inference methods.
Top recommendation systems have used thousands of factors per item and per
user, as was the case in the winning submissions in the Netflix prize \cite{KorBelVol09}.
Recent state-of-the-art methods have relied on learning even larger, more
complex factorization models, often taking nontrivial combinations of multiple
submodels \cite{mackey2011divide,lee2013local}.
Such complex models use large amounts of memory, are increasingly difficult to
interpret, and are often difficult integrate into larger systems. 

\subsection{Linear combinations of attributes}

Our approach is drastically different from previous collaborative
filtering research.  Rather than start with the assumptions of a
matrix factorization model, we make {\it co-clustering} effective for
high quality matrix completion and approximation.  Co-clustering
has been well studied \cite{banerjee2004generalized,shan2008bayesian}
but was not previously competitive in large behavior modeling and
matrix completion problems.  To achieve state of the art results, we
use an \emph{additive model of simple co-clusterings} that we call
stencils, rather than building a large single co-clustering.  The
result is a model that is conceptually simple, has a small parameter
space, has interpretable structure, and achieves the best published 
accuracy for matrix completion on Netflix, as seen in 
Figure \ref{fig:netflix_test_crownjewel}.

Using a linear combination of co-clusterings corresponds to a rather
different interpretation of user preferences and movie
properties. Matrix factorization assumes that a movie preference is
based on a weighted sum of preferences for different genres, with the
movie properties being represented in vectorial form. For instance, if
a user likes comedies but not romantic movies, then a romantic comedy
may have a predicted neutral 3-star rating.

Co-clustering on the other hand assumes there exists some ``correct''
partitioning of movies (and users). For instance, a user might be part
of a group that likes all comedies but does not like romantic
movies. Correspondingly, all romantic comedies might be aggregated
into a cluster, possibly partitioned further into PG-13 rated, or
R-rated romantic comedies. This quickly leads to a combinatorial
explosion.

By taking a linear combination of co-clusterings we benefit from both
perspectives: there is no single correct partitioning of movies and
users; however, we can use the membership in several independent
groups to encode the \emph{factorial} nature of attributes without incurring
the cost of a necessarily high-dimensional model of matrix
factorization. For instance, a movie may be \{funny, sad, thoughtful\}, it might have a
certain age rating, it might be an \{action, romantic, thriller,
documentary, family\} movie, it might be shot in a certain visual
style, and by a certain group of actors. By taking linear combinations
of co-clusterings we can take these attributes into account.

\subsection{Stencils}

The mathematical challenge that motivated this work is that, in order
to encode a rank-$k$ matrix by a factorization, we need $k$ numbers
per row (and column) respectively. With linear combinations of
stencils, on the other hand, we only need $\log_2 k$ bits per row (and
column) plus $O(k^2)$ floating point numbers regardless of the size of
the matrix.

We denote by a stencil a small $k \times k$ template of a matrix and
its mapping to the row and column vectors respectively. This is best
understood by the example below: assume that we have two simple
stencils containing $3 \times 3$ and $3
\times 2$ co-clusterings. Their linear combination yields a rather
nontrivial $9 \times 9$ matrix of rank $5$.
\vspace{2mm}

\noindent
\begin{center}
\includegraphics[width=0.9\columnwidth]{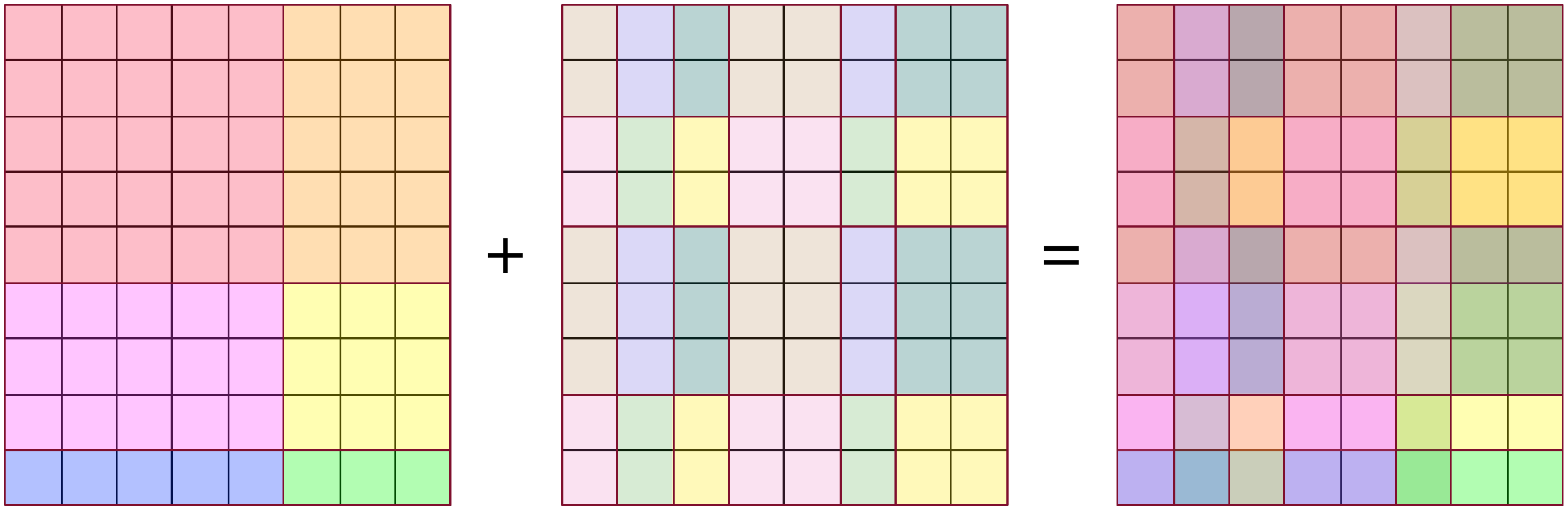}
\end{center}

\vspace{1mm}\noindent%
In contrast, classic co-clustering would require
a $(3 \cdot 3) \times (3\cdot 2)$ partitioning to match this structure.
When we have $s$ stencils of size $k \times k$, this requires a $k^s \times k^s$
partitioning. 

By design our model has a parameter space that is an order of
magnitude smaller than competing methods, requiring only $s\log_2k$
bits per user and per movie and $sk^2$ floating point numbers, where
$k$ is generally quite small.  This is computationally advantageous of
course, but also demonstrates that our modeling assumptions better
match real world structure of human decision making. 

Finding succinct models for binary matrices, e.g.\ by minimizing the
minimum description (MDL), has been the focus of significant research
and valuable results in the data mining community
\cite{van2009identifying,koutravog}. That said, these models are quite
different. To the best of our knowledge, ours is the 
first work aimed at finding a parsimonious model for general (real-valued)
matrix completion and approximation.

\subsection{Contributions}

Our paper makes a number of contributions to the problem of finding
sparse representations of matrices. 
\begin{itemize}
\item We present \method, an iterative $k$-means style algorithm that
  minimizes the approximation error by backfitting the residuals of
  previous approximations. 
\item We provide linear approximation rates exploiting the geometry of
  rows and columns of rating matrix via bounds on the metric entropy
  of Banach spaces.
\item We present a generative Bayesian non-parametric model and devise
  a collapsed Gibbs sampling algorithm, \bmethod, for efficient inference. 
\item Experiments confirm the efficacy of our approach, offering
  the best published results for matrix completion on Netflix, an interpretable
  hierarchy of content, and succinct matrix approximations for ratings, image,
  and network data.
\end{itemize}
We believe that these contributions offer a promising new direction
for behavior modeling and matrix approximation. 

{\bfseries Outline.} We begin by discussing related work from recommendation
systems, non-parametric Bayesian models, co-clustering, and minimum
description length. 
We subsequently introduce the simple $k$-means style co-clustering
and its approximation properties in Section \ref{sec:matapp}.
Subsequently, in Section \ref{sec:single} we define our Bayesian co-clustering
model and collapsed Gibbs sampler for a single stencil.  In
Section~\ref{sec:many} we extend our Bayesian model to multiple stencils. 
Section~\ref{sec:experiments}
reports our experimental results and we conclude with a discussion of future
directions for the work. 

\section{Related Work}

\noindent{\bfseries Recommender Systems.}
Probably the closest to our work is the variety of research on
behavior modeling and recommendation.  Matrix factorization
approaches, such as Koren's SVD++ \cite{Koren08}, have enjoyed great
success in recommender systems.  Recent models such as DFC
\cite{mackey2011divide} and LLORMA \cite{lee2013local} have focused on
using ensembles of factorizations to exploit local structure.

More closely related to our model are Bayesian non-parametric
approaches. For instance, \cite{GriGha11,GorJakRas06} use the Indian
Buffet Process (IBP) for recommendation. In doing so they assume that each
user (and movie) has certain preferential 
binary attributes. It can be seen as an extreme case of \method where
the cluster size $k = 2$, while using a somewhat different strategy to
handle cluster assignment and overall similarity within a
cluster. 
Following a similar intuition as \method but different perspective and focus,
\cite{palla2012infinite} extended the IBP to handle $k > 2$ for link prediction
tasks on binary graphs.
Our work differs in its focus on general, real-valued matrices, its application
of co-clustering, and its significantly simpler parameterization.

A co-clustering approach to recommendation was proposed by
\cite{BeuMurFalSmo14}. This model uses co-clustering to allow
for sharing of strength within each group. However, it does not
overcome the rank-$k$ problem, i.e.\ while clustering reduces
intra-cluster variance and improves generalization, it does not
increase the rank beyond what a simple factorization model is capable
of doing. 
Finally, \cite{PorBarWel08} proposed a factorization model based on a Dirichlet
process over users and columns. All these models are closely related to the
mixed-membership stochastic blockmodels of \cite{AirBleFieXin08}.

\myparagraph{Co-clustering}
It was was originally used primarily for understanding the clustering
of rows and columns of a matrix rather than for matrix approximation
or completion \cite{hartigan1972direct}.  This formulation was well
suited for biological tasks, but it computationally evolved to cover a
wider variety of objectives \cite{banerjee2004generalized}.  
\cite{papalexakis2013k} defined a soft co-clustering objective akin to a
factorization model.
Recent work has defined a Bayesian model for co-clustering focused on matrix
modeling \cite{shan2008bayesian}.  \cite{wang2011nonparametric} focuses on
exploiting co-clustering ensembles, but do so by finding a single consensus
co-clustering.  
As far as we know, ours is the first work to use an additive
combination of co-clusterings.

\myparagraph{Matrix Approximation}
There exists a large body of work on matrix approximation in the
theoretical computer science community. They focus mainly on efficient
low-rank approximations, e.g.\ by projection or by
interpolation. Examples of the projection based strategy are
\cite{HalMarTro09,GitMah13}. Essentially one aims to find a general
low-rank approximation of the matrix, as is common in most recommender
models.

A more parsimonious strategy is to seek \emph{interpolative}
decompositions. There one aims to approximate columns of a matrix by a
linear combination of a subset of other columns
\cite{LiMilPen13}. Nonetheless this requires us to store at least one,
possibly more scaling coefficients per column. Also note the
focus on column interpolations --- this can easily be extended to row
and column interpolations, simply by first performing a row
interpolation and then interpolating the columns. 
To the best of our knowledge, the problem of approximating matrices
with piecewise constant block matrices as we propose here is not the
focus of research in TCS. 

\myparagraph{Succinct modeling}  
The data mining community has focused on finding succinct models of
data, often directly optimizing the model size described by the
minimum description language (MDL) principle
\cite{rissanen1978modeling}.  Finding effective ways to compress real
world data allows for better modeling and understanding of the datasets.  This
approach has led to valuable results in pattern and item-set mining
\cite{vreeken2011krimp,van2009identifying} as well as graph
summarization \cite{koutravog}.  However, these approaches typically
focus on modeling databases of discrete items rather than real-valued
datasets with missing values.

\section{Matrix Approximation}
\label{sec:matapp}

Before delving into the details of Bayesian Non-parametrics we begin
with an optimization view of \method.

\subsection{Model}

Key to our model is the notion of a stencil, an extremely
easy to represent block-wise constant rank-$k$ matrix. 
\begin{definition}[Stencil]
  A stencil $S(T,c,d)$ is a matrix $S \in \RR^{m \times n}$ with the
  property that $S_{ij} = T_{c_i, d_j}$ for a template 
  \smash{$T \in \RR^{k_m \times k_n}$} and discrete index vectors \smash{$c \in \cbr{1,
    \ldots, k_m}^m$} and $d \in \cbr{1, \ldots, k_n}^n$ respectively.
\end{definition}
Given a matrix $M \in \RR^{m \times n}$ it is now our goal to find a
stencil $S(T, c, d)$ such that the approximation error $M - S(T,c,d)$
is small while simultaneously the cost for storing $T, c, d$ is
small. In the context of the example above, the $9 \times 9$ matrix is
given by the sum of two stencils, one of size $3 \times 3$ and one of
size $3 \times 2$. This already indicates that we may require more
than one stencil for efficient approximation. 
In general, our model will be such as to solve
\begin{align}
  \label{eq:lincombmatrix}
  \mini_{\cbr{T^l, c^l, d^l}}\nbr{M - \sum_{l=1}^s S(T^l, c^l, d^l)}^2_{\mathrm{Frob}}
\end{align}
That is, we would like to find an additive model of $s$ stencils that minimizes
the approximation error to $M$.
%
Such an expansion affords efficient compression using a trivial
codebook, as can be seen below.
\begin{lemma}[Compression]
  The stencil $S(T,c,d)$ can be \\ stored at $\epsilon$ element-wise
  accuracy at no more cost than 
  $$O(m \log_2 k_m + n\log_2 k_n + k_m k_n \log_2 \epsilon^{-1} \nbr{T}_\infty).$$
\end{lemma}
\begin{proof} 
This follows directly from the construction. Storing the vector $c$
costs at most $m \log_2 k_m$ bits if we assume a uniform code (this also
holds for $d$). When storing $T$ approximately, we must not quantize at a level
of the approximation error or higher. Hence, a simplistic means of
encoding the dynamic range requires $\log_2 \nbr{T}_\infty/\epsilon$
bit, which is used on a per-element basis in $T$.
\end{proof}
Note that considerably better codes \emph{may} exist whenever the
entropy of $c$ is less than $\log_2 k_m$ (likewise for
$d$). Nonetheless, even the crude $\log_2 k_m$ bound is already much
better than what can be accomplished by a low-rank factorization.

Obviously, given $M$, it is our goal to \emph{find} such stencils
$S(T, c, d)$ with good approximation properties. Unfortunately,
finding linear combinations of co-clusterings is as hard as
co-clustering: assume that we
are given all but one stencil of the optimal solution. In this case
our problem reduces to co-clustering as its subproblem, which is NP-hard.

\subsection{Algorithm}

We consider a simple iterative procedure in which stencils
are computed one at a time, using the residuals as input. The inner
loop consists of a simple algorithm that is reminiscent of $k$-means
clustering. It proceeds in two stages. Without loss of generality we
assume that we have more rows than columns, i.e.\ $M \in \RR^{m \times
  n}$ with $m \geq n$. 

\myparagraph{Row clustering} 
We first perform $k$-means clustering of the rows. That
is, we aim to find an approximation of $M$ that replaces all rows by a
small subset thereof. Note that this is more stringent than the
interpolative approximations of matrices which only require us to find
a set of rows which will form a (possibly sparse) basis for all other
rows. 

\begin{algorithm}[h]
  \caption{RowClustering
  \label{alg:rowcluster}}
  \begin{algorithmic}
    \REQUIRE{matrix $M$, row clusters $k_m$}
    \STATE Draw $k_m$ rows from $M$ at random without replacement and
	copy them to $\cbr{v_1, \ldots, v_{k_m}}$.
    \STATE $t \leftarrow 0 \in \RR^{k_m}$ and $w \leftarrow 0 \in \RR^{k_m
      \times n}$
    \WHILE{not converged}
    \FOR {$i=1$ {\bfseries to} $m$}
    \STATE $c_i \leftarrow \argmin_{j} \nbr{M_{i:} - v_j}_2^2$
    \COMMENT{Find center}
    \STATE $t_{c_i} \leftarrow t_{c_i} + 1$
    \COMMENT{Increment cluster count}
    \STATE $w_{c_i} \leftarrow w_{c_i} + M_{i:}$
    \COMMENT{Increment statistics}
    \ENDFOR
    \FOR{$i=1$ {\bfseries to} $a$}
    \STATE $v_i \leftarrow w_i / t_i$
    \COMMENT{New cluster center}
    \STATE $l_i \leftarrow t_i$ and $t_i \leftarrow 0$
    \COMMENT{Cluster counts}
    \STATE $w_i \leftarrow (0, \ldots 0) \in \RR^n$
    \COMMENT{Reset statistics}
    \ENDFOR
    \ENDWHILE
	\RETURN clusters $\cbr{v_1, \ldots v_{k_m}}$, IDs $c$, counts $l$
  \end{algorithmic}
\end{algorithm}

Algorithm~\ref{alg:rowcluster} is essentially $k$-means clustering on
the rows of $M$. Once we have this, we now cluster the columns of the new
matrix in an analogous manner. The only difference is that the
approximation needs to be particularly good for row clusters that
occur frequently. Consequently the approximation measure $\nbr{M_{i:}
  - v_j}_2^2$ is replaced by the Mahalanobis distance. That is, the
only substantial difference to Algorithm~\ref{alg:rowcluster} is that
now we use the assignment
\begin{align}
  d_i \leftarrow \argmin_{j} \rbr{V_{:i} - w_j}^\top D \rbr{V_{:i} - w_j}
\end{align}
where $V \in \RR^{k_m \times n}$ is the matrix obtained by stacking $V_{i:} = v_i$ and $D$
is the diagonal matrix of counts, i.e.\ $D_{ii} = l_i$. 

\myparagraph{Missing entries}
In many
cases, however, $M$ itself is incomplete. This can be addressed quite
easily by replacing the assignment $\argmin_j \nbr{M_{i:} - v_j}^2_2$
by 
\begin{align}
  c_i \leftarrow \argmin_j \sum_{l: (i,l) \in M} \rbr{M_{il} - v_{jl}}^2
\end{align}
where we used $(i,l) \in M$ as a shorthand for the existing
entries in $M$. In finding a good cluster for the row
$M_{il}$ we restrict ourselves to the coordinates in $v_j$ where
$M_{i:}$ exists.

Likewise, for the purpose of obtaining the column clusters, we now
need to keep track for each coordinate in $V$ how many elements in $M$
contributed to it. Correspondingly denote by $t_{ej} := \sum_{(e,j)
  \in M: c_e = i} 1$ the number of entries mapped into coordinate
$V_{ej}$. Then the assignment for column clusters is obtained via
\begin{align}
  d_j \leftarrow \argmin_l \sum_{i} \rbr{V_{lj} - w_{lj}}^2 t_{lj}
\end{align}
and likewise the averages are now per-coordinate according to the
counts for both $v$ and $w$. 

\myparagraph{Backfitting}
The outcome of row and column clustering is a stencil $S(T,c,d)$
consisting of the clusters obtained by first row and then column
clustering and of the assignment vectors $c$ and $d$ once the process
is complete. It may be desirable to alternate between row and column
clustering for further refinement. Since each step
can only reduce the objective function further and the state space of
$(c, d)$ is discrete, convergence to a local minimum is assured, with
the same caveat on solution quality as in $k$-means clustering.
The last step is to take the residual $M - S(T,c,d)$ and use it as the
starting point of a new approximation round. 

\begin{algorithm}[h]
  \caption{Matrix Approximation
    \label{alg:matapp}}
  \begin{algorithmic}
    \REQUIRE matrix $M$, clusters $k_m, k_n$, max stencils $s$
    \FOR{$\mathrm{iter} = 1$ {\bfseries to} $s$}
    \STATE $(V, c, l_\mathrm{row}) \leftarrow \mathrm{RowClustering}(M, k_a)$
    \STATE $(S, d, l_\mathrm{col}) \leftarrow
    \mathrm{ColumnClustering}(V, b, \mathrm{diag}(l_\mathrm{row}))$
    \FOR{{\bfseries all} $i, j \in \cbr{1, \ldots k_m} \times \cbr{1, \ldots k_n}$}
    \STATE $T_{ij} \leftarrow \mathrm{mean} \cbr{M_{ef} | c_e = i
      \text{ and } d_f = j}$
    \ENDFOR
    \STATE $M \leftarrow M - S(T,c,d)$ and back up $(T,c,d)$
    \ENDFOR
  \end{algorithmic}
\end{algorithm}

Essentially the last stage is used to ensure that the stencil has
minimum approximation error given the partitioning. This procedure is
repeatedly invoked on the residuals to minimize the loss.  The result is an
additive model of co-clusterings.

\subsection{Approximation Guarantees}
\label{sec:entropy}

A key question is how well any given matrix $M$ can be approximated by
an appropriate stencil. For the sake of simplicity we limit ourselves
to the case where all entries of the matrix are observed. We use
covering numbers and spectral properties of $M$ to obtain approximation
guarantees. 
\begin{definition}[Covering Number]
  Denote by $\Bcal$ a Banach Space. Then for any given set $B \in
  \Bcal$ the covering number $\Ncal_\epsilon(B)$ is given by the set
  of points $\cbr{b_1, \ldots b_{\Ncal_\epsilon(B)}}$ such that for
  any $b \in B$ there exists some $b_j$ with $\nbr{b - b_j} \leq \epsilon$.
\end{definition}
Of particular interest for us are covering numbers $\Ncal_\epsilon$ of
unit balls and their functional inverses $\epsilon_n$. The latter are
referred to as entropy number and they quantify the approximation
error incurred by using a cover of $n$ elements \cite[Chapter 8]{Smola98}.
A key tool for computing entropy numbers of scaling
operators is the theorem of Gordon, K\"onig and Sch\"utt
\cite{GorKonSch87}, relating entropy numbers to singular values.
\begin{theorem}[Entropy Numbers and Singular Values]
  \label{th:svd}
  Denote by $D$ a diagonal scaling operator with $D: \ell_p \to
  \ell_p$ with scaling coefficients $\sigma_i \geq \sigma_{i+1} \geq
  0$ for all $i$. Then for all $n \in \NN$ the entropy number
  $\epsilon_n(D)$ is bounded via
  \begin{align}
    \label{eq:gks}
    \epsilon_n(D) \leq
    6 \sup_{j \in \NN} \rbr{n^{-1} \prod_{i=1}^j \sigma_i}^{\frac{1}{j}}
    \leq 6 \epsilon_n(D) 
  \end{align}
\end{theorem}
This means that if we have a matrix with rapidly decaying singular
values, we only need to focus on the leading largest ones in order to
approximate all elements in the space efficiently. Here the tradeoff
between dimensionality and accuracy is obtained by using the harmonic
mean.
\begin{corollary}[Entropy Numbers of Unit Balls]
  \label{cor:ball}
  The covering number of a ball $\Bcal$ of radius $r$ in $\ell_2^d$ is
  bounded by 
  \begin{align}
    r n^{-\frac{1}{d}} \leq \epsilon_n(\Bcal) \leq 6 r n^{-\frac{1}{d}}.
  \end{align}
\end{corollary}
\begin{proof} 
This follows directly from using the linear operator $x \to r x$ for
$x \in \ell_2^d$. Here the scaling operator has eigenvalues $\sigma_i
= r$ for all $1 \leq i \leq d$ and $\sigma_i = 0$ for $i > d$. 
The maximum in \eq{eq:gks} is always $j = d$. 
\end{proof}
The following theorem states that we can approximate $M$ up to a
multiplicative constant at any step, provided that we pick a large
enough clustering. It also means that we get linear convergence, i.e.\
convergence in $O(\log \epsilon)$ steps to $O(\epsilon)$ error, since
the bound can be applied iteratively.
\begin{theorem}[Approximation Guarantees]
  Denote by $\sigma_1, \ldots \sigma_n$ the singular values of $M$. Then
  using $l$ clusters for rows and columns respectively the matrix $M$
  can be approximated with error at most 
  \begin{align*}
    \nbr{M - M'}_\infty & \leq 2 \nbr{M}^{\frac{1}{2}} \epsilon_l\rbr{\Sigma^{\frac{1}{2}}}
    \\
    \nbr{M - M'}_2 & \leq (\sqrt{m} + \sqrt{n}) \nbr{M}^{\frac{1}{2}}
                     \epsilon_l\rbr{\Sigma^{\frac{1}{2}}} 
  \end{align*}
  Here $\epsilon_l$ is given by Theorem~\ref{th:svd} and
  Corollary~\ref{cor:ball} respectively. 
\end{theorem}
\begin{proof}
  Using the singular value decomposition of $M$ into $M = U \Sigma V$ we can
  factorize $M = Q^\top R$ where $Q = \Sigma^{\frac{1}{2}} U$ and $R =
  \Sigma^{\frac{1}{2}} V$. By construction, the singular values of $Q$
  and $R$ are $\Sigma^{\frac{1}{2}}$. We now cluster the rows of $Q$ and $R$
  independently to obtain an approximation of $M$. 

  For $Q$ we know that its rows can be approximated by $l$ balls with
  error $\epsilon_l\rbr{\Sigma^{\frac{1}{2}}}$ as per
  Theorem~\ref{th:svd}. Also note that its row vectors are contained
  in the image of the unit ball under $Q$ --- if they were not,
  project them onto the unit ball and the approximation error cannot
  increase since the targets are within the unit ball, too. Hence the
  $\epsilon_l$-cover of the latter also provides an approximation of
  the row-vectors of $Q$ by $Q'$ with accuracy $\epsilon_l$, where
  $Q'$ contains at most $l$ distinct rows. The same holds for the
  matrix $R$, as approximated by $R'$. Hence we have
  \begin{align*}
    \abr{M_{ij} - \inner{Q_{i:}'}{R_{j:}'}}
    & = \abr{\inner{Q_{i:}}{R_{j:}} - \inner{Q_{i:}'}{R_{j:}'}} \\
    & = \abr{\inner{Q_{i:} - Q_{i:}'}{R_{j:}} + \inner{Q_{i:}'}{R_{j:} - R_{j:}'}} \\
    & \leq \nbr{Q_{i:} - Q_{i:}'}\nbr{R_{j:}} +
    \nbr{Q_{i:}'}\nbr{R_{j:} - R_{j:}'} \\
    & \leq 2 \nbr{M}^{\frac{1}{2}} \epsilon_l\rbr{\Sigma^{\frac{1}{2}}}
  \end{align*}
  This provides a \emph{pointwise} approximation guarantee.If we only
  have a bound on the rank and on $\nbr{M}$, this yields
  \vspace{-4mm}
  \begin{align*}
    \abr{M_{ij} - \inner{q_i}{r_j}} \leq
    12 \nbr{M} l^{\frac{1}{2d}}
  \end{align*}
  Moreover, since each row in $Q$ and $R$ respectively will be
  approximated with residual bounded by $\epsilon_l$ we can bound 
  $\nbr{Q - Q'} \leq \sqrt{m} \epsilon_l$ and 
  $\nbr{R - R'} \leq \sqrt{n} \epsilon_l$ respectively. This yields a
  bound on the matrix norm of the residual via 
  \begin{align*}
    \nbr{\sbr{QR - Q' R'}x} & \leq
    \nbr{Q (R-R')x} +  \nbr{(Q-Q') R'x} \\
    & \leq (\sqrt{m} + \sqrt{n}) \nbr{M}^{\frac{1}{2}}
      \epsilon_l\rbr{\Sigma^{\frac{1}{2}}}\nbr{x} 
  \end{align*}
  This bounds the matrix norm of the residual. 
\end{proof}
Note that the above is an \emph{existence} proof rather than a
constructive prescription. However, by using the fact that set cover
is a submodular problem \cite{Wolsey82}, it follows that given $l$, we are
able to obtain a near-optimal cover, thus leading to a
\emph{constructive} algorithm. Note, however, that the main purpose of
the above analysis is to obtain theoretical upper bounds on the rate of
convergence. In practice, the results can be considerably better, as
we show in Section~\ref{sec:experiments}.

\section{Generative Model}

In the same manner as many risk minimization problems (e.g.\ penalized
logistic regression) have a Bayesian counterpart (Gaussian Process
classification) \cite{Neal94}, we now devise a Bayesian
counterpart to \method, which we will refer to as 
\bmethod. We begin with the single stencil case in
Section~\ref{sec:single} and extend it to many stencils in
Section~\ref{sec:many}. 

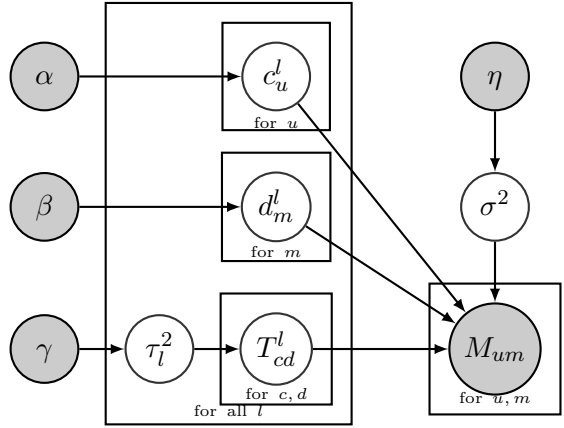
\begin{figure}[bt]
	\centering
\begin{tikzpicture}[>=latex,text height=1.5ex,text depth=0.25ex]
  \matrix[row sep=0.8cm,column sep=0.6cm] {
    \node (alpha) [observed]{$\alpha$}; & &
    \node (c) [latent]{$c_u^l$}; & & &
    \node (gamma) [observed]{$\eta$}; &
    \\
    \node (beta) [observed]{$\beta$}; & & 
    \node (d) [latent]{$d_m^l$}; & & &
    \node (sigma) [latent]{$\sigma^2$}; &
    \\
    \node (gammatau) [observed]{$\gamma$}; &
    \node (tau) [latent]{$\tau_l^2$}; &
    \node (S) [latent]{$T_{cd}^l$}; & & &
    \node (r) [observed]{$M_{um}$}; &
    \\
   };
  \path[->]
  (gammatau) edge[thick] (tau)
  (tau) edge[thick] (S)
  (alpha) edge[thick] (c)
  (beta) edge[thick] (d)
  (c) edge[thick] (r)
  (d) edge[thick] (r)
  (S) edge[thick] (r)
  (sigma) edge[thick] (r)
  (gamma) edge[thick] (sigma)
  ;
  \begin{pgfonlayer}{background}
    \node (usercluster) [plate, fit=(c)] {\
      \\[9mm]\tiny for $u$};
    \node (moviecluster) [plate, fit=(d)] {\
      \\[8.5mm]\tiny for $m$};
    \node (stencil) [plate, fit=(S)] {\
      \\[9mm]\tiny for $c,d$};
    \node (ratings) [plate, fit=(r)] {\
      \\[10mm]\tiny for $u,m$};
    \node (stencils) [plate, fit=(usercluster) (moviecluster) (tau) (stencil)] {\ 
      \\[49mm]\tiny for all $l$};
    \end{pgfonlayer}
\end{tikzpicture}
\caption{Generative model for recommdnation and matrix approximation
  (\bmethod). For each stencil, as indexed by $l$, row and cluster
  memberships $c^l$ and $d^l$ are drawn from a Chinese Restaurant
  Process. The values for the template $T^l$ are drawn from a Normal
  Distribution. The observed ratings $M_{um}$ are 
  sums over the stencils $S(T^l, c^l, d^l)$. \label{fig:multi_stencil_model}}
\end{figure}

\subsection{Co-Clustering with a Single Stencil}
\label{sec:single}

We begin with a simplistic model of co-clustering. It serves as the
basic template for single-matrix inference. All subsequent steps use
the same idea. In a nutshell, we assume that each user $u$ belongs to
a particular cluster $c_u$ drawn from a Chinese Restaurant Process
$\mathrm{CRP}(\alpha)$. Likewise, we assume that each movie $m$
belongs to some cluster $d_m$ drawn analogously from
$\mathrm{CRP}(\beta)$. The scores of the matrix $M_{um}$ are obtained
from a stencil $S(T,c,d)_{um} = T_{c_u d_m}$ with additive noise
$\epsilon_{um} \sim \Ncal(0, \sigma^2)$. The stencil values $T_{cd}$
themselves are drawn from a normal distribution $\Ncal(0, \tau^2)$. In
turn, the variances $\tau^2$ and $\sigma^2$ are obtained via a
conjugate prior, i.e.\ the Inverse Gamma distribution.

This is an extremely simple model similar to \cite{shan2008bayesian}, akin
to a decision stump. The
rationale for picking such a primitive model is that we will be combining
linear combinations thereof to obtain a very flexible tool.
The model is shown in Figure \ref{fig:multi_stencil_model}. For $l =
1$ the formal definition is as follows: 
\begin{subequations}
\begin{align}
  c_u & \sim \mathrm{CRP}(\alpha) & 
  d_m & \sim \mathrm{CRP}(\beta)  
  \\
  T_{cd} & \sim \Ncal(0, \tau^2) &
  M_{um} & \sim \Ncal\rbr{T_{c_u d_m}, \sigma^2}
  \\
  \tau^2 & \sim \mathrm{IG}(\gamma) &
  \sigma^2 & \sim \mathrm{IG}(\eta) 
\end{align}
\end{subequations}
Recall that the Inverse Gamma distribution is given by 
\begin{align}
  \label{eq:igitt}
  p(x|a, b) & = b^a \Gamma^{-1}(a) x^{-a-1} e^{-\frac{b}{x}} \\
  \text{and hence }
  p(\sigma^2|\eta_a, \eta_b) & = \eta_b^{\eta_a}
  \Gamma^{-1}(\eta_a) \sigma^{-2(\eta_a + 1)} e^{-\frac{\eta_b}{\sigma^2}}
\end{align}
Consequently the joint probability distribution over all scores, given
the variances is given by
\begin{align}
  & p\rbr{M, S(T,c, d)|\alpha, \beta, \sigma^2, \tau^2} 
  \label{eq:fullmodel} \\
  = & \mathrm{CRP}(c|\alpha) \mathrm{CRP}(d|\beta) 
      \prod_{c,d} { \frac{1}{\sqrt{2 \pi \tau^2}}}
      \exp\rbr{-\frac{T_{cd}^2}{2 \tau^2}}
      \prod_{(u,m)} { \frac{1}{\sqrt{2 \pi \sigma^2}}}
      \exp\rbr{-\frac{\rbr{M_{um} - T_{c_u d_u}}^2}{2 \sigma^2}} \nonumber
\end{align}
The idea is that each user and each movie are characterized by a
simple cluster. Since we chose all priors to be conjugate to the
likelihood terms, it is possible to collapse out the choice of
$T_{cd}$. 
This is particularly useful as it
allows us to accelerate the sampler considerably --- now we only
sample over the discrete random variables $c_u, d_m$ indicating the
cluster memberships for a particular user and movie. In other words,
we obtain a closed form expression for $p(M, c, d|\alpha, \beta,
\sigma^2, \tau^2)$. Moreover, $\sigma^2|S(T,c,d), \eta$ and $\tau^2|T,
\gamma$ are both Inverse Gamma due to conjugacy, hence we can sample
them efficiently after sampling $T$. 

\subsection{Inferring Clusters}

In the following we discuss a partially collapsed Gibbs sampler
(effectively we use a Rao-Blackwellization strategy when sampling
cluster memberships) for efficient inference. We begin with the part
of sampling $c|d, \alpha, \sigma^2, \tau^2$.
\begin{description}
\item[Chinese Restaurant Process:] It is well known that for
  exponential families the conjugate distribution allows for
  collapsing by taking ratios between normalization coefficients with
  and without the additional sufficient statistics item. See e.g.\
  \cite[Appendix A]{GriGha11} for a detailed derivation. Denote by
  $n_i,m_j$ the number of users and movies belonging to 
  clusters $i$ and $j$ respectively. Moreover, denote by $n$ and $m$ the
  total number of users and movies, and by $k_n, k_m$ the number of
  clusters. In this case we can express 
  \begin{align}
    \nonumber
	p(c|\alpha) & = \alpha^{k_n} \sbr{\prod_{i=1}^{k_n} \Gamma(n_i)}
                  \frac{\Gamma(\alpha)}{\Gamma(\alpha + n)} \\
    \text{and hence }
    \nonumber 
    p(c_i = t|c^{-i}, \alpha) & = 
    \begin{cases}
      \frac{n_t^{-i}}{\alpha + n - 1} & \text{ if $n_t^{-i} > 0$} \\
      \frac{\alpha}{\alpha + n-1} & \text{ otherwise}
    \end{cases}
  \end{align}
  An analogous expression is available for $p(d|\beta)$ and $p(d_j =
  t|d^{-j},\beta)$. Note that the superscript $\sbr{\cdot}^{-i}$ denotes
  that the $i$-th observation is left out when computing the
  statistic. Large values of $\alpha$ and $\beta$ encourage the
  formation of larger numbers of clusters. The collapsed expressions
  will be useful for Gibbs sampling.
\item[Integrating out $T$:]
For faster mixing we need to integrate out $T$ whenever we resample
$c$ and $d$. As we shall see, this is easily accomplished by keeping
simple linear statistics of the ratings. Moreover, by integrating out
$T$ we avoid the problem of having to instantiate a new value whenever
a new cluster is added. 

For a given block $(c,d)$ with associated $T_{cd}$, the distribution
of ratings is Gaussian with mean $0$ and with covariance matrix
$\sigma^2 \one + \tau^2 1 1^\top$ (due to the independence of the variances and the additive nature of the normal distribution). Here we use $\one$ to denote
the identity matrix and $1$ to denote the vector of all $1$. 
%
Denote by $n_{cd}$ the number of rating pairs $(u,m)$ for which $c_u =
c$ and $d_m = d$. Moreover, denote by $M_{cd}$ the vector of
associated ratings. Hence, the likelihood of the cluster block
$(c,d)$, as observed in $M_{cd}$ is 
\begin{align}
  p(M_{cd}|\sigma^2, \tau^2) = \frac{\exp\sbr{-\frac{1}{2} M_{cd}^\top
  \sbr{\sigma^2 \one + \tau^2 1 1^\top}^{-1} M_{cd}}}{(2 \pi)^{\frac{n_{cd}}{2}}
  \abr{\sigma^2 \one + \tau^2 1 1^\top}^{\frac{1}{2}}}  
  \nonumber
\end{align}
In computing the above expression we need to compute the determinant
of a diagonal matrix with rank-1 update, and the inverse of said
matrix. For the former, we use the matrix-determinant lemma, and for
the latter, the Sherman-Morrison-Woodbury formula:
\begin{align*}
  \hspace{-10mm}
  M_{cd}^\top \sbr{\sigma^2 \one + \tau^2 1 1^\top}^{-1} M_{cd} & = 
  \frac{1}{\sigma^{2}} \nbr{M_{cd}}^2  -
  \frac{\tau^2}{\sigma^2} \cdot \frac{\rbr{1^\top M_{cd}}^2}{\sigma^2 + n_{cd}
    \tau^2} \\
  \hspace{-10mm}
  \log \abr{\sigma^2 \one + \tau^2 1 1^\top}  & = (n_{cd} - 1)
  \log \sigma^2 + \log \sbr{\sigma^2 + n_{cd} \tau^2}
\end{align*}
This allows us to assess whether it is beneficial to assign a user $u$
or a movie $m$ to a different or a new cluster efficiently, since the
only statistics involved in the operation are sums of residuals
and of their squares. 
\end{description}
This leads to a collapsed Gibbs-sampling algorithm. At
each step we check how likelihoods change by assigning a movie (or
user) to another cluster. We denote by $n_{cd}'$ the new cluster count
and by $M_{cd}'$ the new set of residuals. Let
\begin{align}
  \Delta := {\textstyle \frac{n_{cd}' - n_{cd}}{2}} \sbr{\log (2\pi) +
  \log \sigma^2} + \frac{1}{2 \sigma^2}
  \sbr{\nbr{M_{cd}'}^2 - \nbr{M_{cd}}^2}
  \nonumber
\end{align}
be a constant offset, in log-space, that only depends on the additional ratings that
are added to a cluster. In other words, it is \emph{independent} of the cluster
that the additional scores are assigned to. Hence $\Delta$ can be safely
ignored.
\begin{align}
  \label{eq:sampleold}
  p(c_u = c|\cdot) \propto & \frac{n_c^{-i}}{\alpha + n-1}
  \prod_d \sbr{\frac{\sigma^2 + n_{cd} \tau^2}{\sigma^2 + n_{cd}'
  \tau^2}}^{\frac{1}{2}} 
  \exp\sbr{\frac{\tau^2}{2 \sigma^2} \sum_d \sbr{ 
    \frac{\rbr{1^\top M_{cd}'}^2}{\sigma^2 + n_{cd}' \tau^2} -
    \frac{\rbr{1^\top M_{cd}}^2}{\sigma^2 + n_{cd} \tau^2}}}
\end{align}
For a new cluster this can be simplified since there is no data, hence
$n_{cd} = 0$ and $M_{cd} = []$.
\begin{align}
  \label{eq:samplenew}
  p(c_u = c_\mathrm{new}|\cdot) 
  \propto & {\textstyle\frac{\alpha}{\alpha + n-1}}
  \prod_d \sbr{\textstyle\frac{\sigma^2}{\sigma^2 + n_{cd}' \tau^2}}^{\frac{1}{2}} 
  \exp\sbr{{\textstyle\frac{\tau^2}{2 \sigma^2}} \sum_d 
  \textstyle\frac{\rbr{1^\top M_{cd}'}^2}{\sigma^2 + n_{cd}' \tau^2}}
\end{align}
The above expression
is fairly
straightforward to compute: we only need to track $n_{cd}$, i.e.\ the
number of ratings assigned to a particular (user cluster, movie
cluster) combination and $1^\top M_{cd}$, i.e.\ the sum of the scores
for this combination. 

\subsection{Inferring Variances}

For the purpose of recommendation and for a subsequent combination of
several matrices, we need to infer variances and instantiate the scores
$T_{cd}$. By checking \eq{eq:fullmodel} we see that $T_{cd}|\mathrm{rest}$ is
given by a normal distribution with parameters
\begin{align}
  \label{eq:scd-normal}
  T_{cd}|\mathrm{rest} \sim \Ncal\rbr{\textstyle \frac{1^\top
      M_{cd}}{\rho n_{cd}}, \frac{\sigma^2}{\rho n}}
  \text{ where }  \rho =
  \sbr{1 + \textstyle \frac{1}{n_{cd}} \frac{\sigma^2}{\tau^2}}.
\end{align}
Note that the term $\rho$ plays the role of a classic shrinkage term
just as in a James-Stein estimator. To sample $\sigma^2$ and $\tau^2$
we use the Inverse Gamma distribution of \eq{eq:igitt}. 

Denote by $E$ the total number of observed values in $M$. 
In this case, $\sigma^2$ is drawn from an
Inverse Gamma prior with parameters $(\eta'_a, \eta'_b)$:
\begin{align}
  \label{eq:ig-eta}
  \eta'_a \leftarrow \eta_a + \frac{E}{2}
  \text{ and }
  \eta'_b \leftarrow \eta_b + \sum_{(u,m)} (M_{um} - T_{c_u d_m})^2.
\end{align}
Analogously, we draw $\tau^2$
from an Inverse Gamma with parameters 
\begin{align}
  \label{eq:ig-gamma}
  \gamma'_a \leftarrow \gamma_a + \frac{k_n k_m}{2}
  \text{ and }
  \gamma'_b \leftarrow \gamma_b + \sum_{c,d} T_{cd}^2
\end{align}
$k_n$ and $k_m$ denote the number of user and movie clusters.

\subsection{Efficient Implementation}

With these inference equations we can implement an efficient sampler, as seen in
Algorithm \ref{alg:sampler}.
The key to efficient sampling is to cache the per-cluster sums of ratings
$1^\top M_{cd}$. 
Then reassigning a user (or
movie) to a different (or new) cluster is just a matter of checking
the amount of change that this would effect.
Hence each sampling pass costs $O(k_n \cdot k_m \cdot (n+m) + E)$ operations. It
is linear in the number of ratings and of partitions. 

\begin{algorithm}[htb]
  \caption{StencilSampler \label{alg:sampler}}
\begin{algorithmic}
  \STATE {\bfseries Initialize} row-index and column-index of data in $M$
  \STATE {\bfseries Initialize} sum of squares $Q := \sum_{(u,m)} M_{um}^2$
  \STATE {\bfseries Initialize} statistics for each partition
  \begin{align*}
    n_{cd} := \cbr{(u,m): c_u = c, d_m = d} 
    \text{ and }
    l_{cd} := \hspace{-7mm}\sum_{(u,m): c_u = c, d_m = d}\hspace{-7mm} M_{um}
  \end{align*}
  \WHILE{sampler not converged}
  \FOR{{\bfseries all users} $u$}
  \STATE For all movie clusters $d$ compute the incremental changes
  \begin{align*}
    n_{ud} := \cbr{(u,m): d_m = d} \text{ and }
    l_{ud} := \hspace{-3mm}\sum_{(u,m): d_m = d}\hspace{-3mm} M_{um}
  \end{align*}
  \STATE Remove $u$ from their cluster
  $\displaystyle n_{c_u d} \leftarrow n_{c_u d} - n_{ud}$ 
  \STATE Remove $u$ from their cluster
  $\displaystyle l_{c_u d} \leftarrow l_{c_u d} - n_{ud}$
  \STATE Sample new user cluster $c_u$ using \eq{eq:sampleold} and \eq{eq:samplenew}.
  \STATE Update statistics 
  \begin{align*}
  n_{c_u d} \leftarrow n_{c_u d} + n_{ud}
  \text{ and }
  l_{c_u d} \leftarrow l_{c_u d} + n_{ud}
  \end{align*}
  \ENDFOR
  \FOR{{\bfseries all movies} $m$}
  \STATE Sample movie cluster assignments analogously.
  \ENDFOR
  \FOR{{\bfseries all} $(c,d)$ cluster partitions}
  \STATE Resample $T_{cd}$ using \eq{eq:scd-normal} and the statistics
  $n_{cd}, l_{cd}$.
  \ENDFOR
  \STATE Resample $\sigma^2$ and $\tau^2$ 
  via the Inverse Gamma distribution using \eq{eq:ig-eta} and \eq{eq:ig-gamma}.
  \ENDWHILE
\end{algorithmic}
\end{algorithm}
%
Note that once $n_{ud}$ and $l_{ud}$ are available for all users (or
all movies), it is cheap to perform additional sampling sweeps at
comparably low cost. It is therefore beneficial to iterate over
all users (or all movies) more than once, in particular in the initial
stages of the algorithm. 
Also note that the algorithm can be used on datasets that are being
streamed from disk, provided that an index and an inverted index of
$M$ can be stored: we need to be able to traverse the data
when ordered by users and when ordered by movies. It is thus
compatible with solid state disks.

\subsection{Additive Combinations of Stencils}
\label{sec:many}


If there was no penalty on the number of clusters it would be possible to
approximate any matrix by a trivial model using as many
clusters as we have rows and columns.
\begin{lemma}
  Any matrix $M \in \RR^{m \times n}$ has nonvanishing support in
  \eq{eq:fullmodel} regardless of $\sigma^2$. 
\end{lemma}
\begin{proof}
  Since any partitionings of sets of size $m, n$ respectively have
  nonzero support, it follows that partitioning all rows and all
  columns into separate bins is possible. Hence, we can
  assign a different mean $\mu_{cd}$ to any entry $M_{um}$. 
\end{proof}
Obviously, the CRP prior on $c$ and $d$ makes this highly unlikely. On
the other hand, we want to retain the ability to fit a richer set of
matrices than what can be effectively covered by piecewise constant
block matrices.  We take linear combinations of
matrices, as introduced in Section~\ref{sec:matapp}. 

As before, we enumerate the stencils by $S(T^l, c^l,
d^l)$. Correspondingly we now need to sample from a set of $S(T^l,
c^l, d^l)$ and $\tau^2$ \emph{per} matrix. However, we keep the
additive noise term $\Ncal(0, \sigma^2)$ unchanged. This is the model
of Figure \ref{fig:multi_stencil_model}.
The additivity of
Gaussians renders makes inference easy:
\begin{align}
  M_{um} \sim \Ncal\rbr{\sum_l S(T^l, c^l, d^l), \sigma^2}.
\end{align}
Note, though, that estimating $S$ jointly \emph{for all} indices $l$
is not tractable since various clusterings $(c^l, d^l)$ overlap and
intersect with each other, hence the joint normal distribution over all
variables would be expensive to factorize.
\begin{algorithm}[tb]
  \caption{bACCAMS \label{alg:baccams}}
\begin{algorithmic}
  \STATE {\bfseries initialize} residuals $\rho := M$ and $T^l = 0$
  \WHILE{sampler not converged}
  \FOR{{\bfseries all stencils} $l$}
  \STATE Compute partial residuals $\rho \leftarrow \rho - S(T^l, c^l,d^l)$
  \STATE Sample over $S(T^l, c^l, d^l)$ using $\rho$ instead of $M$
  \STATE Update residuals with $\rho \leftarrow \rho + S(T^l, c^l, d^l)$
  \ENDFOR
  \ENDWHILE
\end{algorithmic}
\end{algorithm}
Instead, we sample over one stencil at a time, as shown in Algorithm \ref{alg:baccams}.
This algorithm only requires repeated passes through the
dataset. Moreover, it can be modified into a backfitting procedure by
fitting one matrix at a time and then fixing the outcome. Capacity
control can be enforced by modifying $\alpha$ and $\beta$ such that
the probability of a new cluster decreases for larger $l$, i.e.\ by
decreasing $\alpha$ and $\beta$.  As a result following the analysis
in the single stencil case, each sampling pass costs
$O(s\cdot(k_n \cdot k_m \cdot (n+m) + E))$ operations where $s$ is the number
of stencils. It is linear in the number of ratings, in the number of
partitions and in the number of stencils.

\vspace{3mm}  

\section{Experiments}
\label{sec:experiments}

We evaluate our method based on its ability to perform matrix
completion, matrix approximation and to give interpretable results.
Here we describe our experimental setup and results on real
world data, such as the Netflix ratings.

\subsection{Implementation}

We implemented both \method, the $k$-means-based algorithm, as well as
\bmethod, the Bayesian model. Unless specified otherwise, we run the RowClustering
of Algorithm~\ref{alg:rowcluster} for up to $T=50$ iterations. Our
system can also iterate over the stencils multiple times, such that earlier
stencils can be re-learned after we have learned later ones.  In practice, we
observe this only yields small gains in accuracy, hence we generally do
not use it.

We implemented \bmethod using Gibbs sampling (Section~\ref{sec:many}) and used
the $k$-means algorithm \method for
the initialization of each stencil. Following standard practice, we
bound the range of $\sigma$ by $\sigma_{\rm max}$ from above. This
rejection sampler avoids pathological cases. For the
sake of simplicity, we set $k = k_n = k_m$ to be the maximum number of clusters
that can be generated in each stencil.  When inferring the cluster
assignments for a given stencil, we run three iterations of the
sampler before proceeding to the next stencil.  As common in MCMC
algorithms, we use a burn-in period of at least 30 iterations (each
with three sub-iterations of sampling cluster assignments) and 
then average the predictions over many draws.
Code for both \method and \bmethod is available at
\url{http://alexbeutel.com/accams}.

\subsection{Experimental Setup} 
\label{sub:setup}

\myparagraph{Netflix}
We run our algorithms on data from a variety of domains.  Our primary testing
dataset is the ratings dataset from the Netflix contest.  The
dataset contains 100M ratings from 480k users and 17k movies.
Following standard practice for testing recommendation accuracy, we
average over three different random 90:10 splits for training and testing. 


\myparagraph{CMU Face Images}
To test how well we can approximate arbitrary matrices, we use image
data from the CMU Face Images
dataset\footnote{http://goo.gl/FsoX5p}.
It contains black and white images of 20 different people, each in 32 different
positions, for a total of 640 images.  Each
image has $128 \times 120$ pixel resolution; we flatten this into a
matrix of $640 \times 15360$, i.e.\ an image by pixel matrix.

\myparagraph{AS Peering Graph}
To assess our model's ability to deal with graph data we consider the AS 
graph\footnote{http://topology.eecs.umich.edu/data.html}.  
It contains information on the peering information
of 13,580 nodes. It thus creates a binary matrix of size $13,580 \times
13,580$ with 37k edges. Since our algorithm is not designed to
learn binary matrices, we treat the entries $\cbr{0, 1}$ as real
valued numbers.

\myparagraph{Parameters}
For all experiments, we set the hyperparameters in \bmethod to
$\alpha=\beta=10$, $\eta_\alpha = 2$, $\eta_\beta = 0.3$,
$\gamma_\alpha = 5$, and $\gamma_\alpha=0.3$.  Depending on the task,
we compare \method against SVD++ using the GraphChi \cite{KyrBleGue12}
implementation, SVD from Matlab for full matrices, and 
previously reported state-of-the-art results.

\myparagraph{Model complexity} 
Since our model is structurally quite different from 
factorization models, we compare them based on 
the number of bits in the model and prediction accuracy.  For factorization
models, we consider each factor to be a 32 bit {\tt float}.  Hence the
complexity of a rank $R$ SVD++ model of $n$ users and $m$ movies 
is $32\cdot R(n+m)$ bits. 

For \method with $s$ stencils and $k \times k$
co-clusters in each stencil, the cluster assignment for a given row or column
is $\log_2k$ bits and each value in the stencil is a {\tt float}.  As
such, the complexity of a model is $s ((n+m) \log_2 k + 32\cdot k^2)$
bits.

In calculating the parameter space size for LLORMA, we make the very            
conservative estimate that each row and column is on average part of two        
factorizations, even though the model contains more than 30 factorizations that 
each row and column could be part of.                                           

\subsection{Matrix Completion} 
\label{sub:Recommendation}

Since the primary motivation of our model is collaborative filtering
we begin by discussing results on the classic Netflix
problem; accuracy is measured in RMSE. To avoid divergence we set
$\sigma_{\rm max} = 1$. We then vary both the number of 
clusters $k$ and the number of stencils $s$.  

A summary of recent results as well as results using our method can be
found in Table \ref{tab:netflix}.  Using GraphChi we run SVD++
on our data. 
We use the reported values from LLORMA
\cite{lee2013local} and DFC \cite{mackey2011divide},
which were obtained using the same protocol as reported here.

As can be seen in Table \ref{tab:netflix}, \bmethod achieves the {\it best}
published result.
We achieve this while using 
a very different model that is significantly simpler both
conceptually and in terms of parameter space size. We also did not use
any of the temporal and contextual variants that many other models use
to incorporate prior knowledge. 

As shown in Figure
\ref{fig:netflix_test_crownjewel}, we observe that per bit our model achieves
much better accuracy at a fraction of the model size. 
In Figure \ref{fig:matrix_approx}(a) we compare different
configurations of our algorithm.  As can be seen, classic
co-clustering quickly overfits the training data and provides a less
fine-grained ability to improve prediction accuracy than \method.  
Since \method has no regularization, it too overfits the training data.
By using a Bayesian model with \bmethod, we do not
overfit the training data and thus can use more stencils for prediction,
greatly improving the prediction accuracy.

\begin{table}[tb]
	\centering
	\begin{tabular}{l|l|l|l}
		Method & Parameters & Size & Test RMSE \\ \hline
		SVD++ \cite{KyrBleGue12} & $R=25$ & 49.8MB & 0.8631 \\ 
		DFC-NYS \cite{mackey2011divide} & \multicolumn{2}{|c|}{Not reported} & 0.8486 \\ 
		DFC-PROJ  \cite{mackey2011divide} & \multicolumn{2}{|c|}{Not reported} & 0.8411 \\ 
		LLORMA \cite{lee2013local} & $R=1$ & 3.98MB & 0.9295 \\ 
		LLORMA \cite{lee2013local}  & $R=5$, $a > 30$ &19.9MB & 0.8604 \\ 
		LLORMA \cite{lee2013local}  & $R=10$, $ a> 30$ & 39.8MB & 0.8444 \\ 
		LLORMA \cite{lee2013local}  & $R=20$, $a > 30$& 79.7MB & 0.8337 \\ 
		\hline
		\method & $k=10$, $s=13$ &2.69MB & 0.8780 \\ 
		\method & $k=100$, $s=5$ &2.27MB & 0.8759 \\ 
		\bmethod & $k=10$, $s=50$ &10.4MB & 0.8403  \\ 
		\bmethod & $k=10$, $s=70$ &14.5MB& 0.8363  \\ 
		{\bfseries \bmethod} & $k=10$, $s=125$ & {\bfseries 25.9MB} & {\bfseries 0.8331  }
	\end{tabular}
	\caption{\bmethod achieves an accuracy for matrix completion on Netflix better than or
	on-par with the best published results, while having a parameter space a
	fraction of the size of other methods.  $a$ denotes the number of anchor
	points for LLORMA and sizes listed are the parameter space size.}
	\label{tab:netflix}
\end{table}

\begin{figure}[tb]
	\centering
 \begin{tabular}{@{}cc@{}}
  \includegraphics[width=0.37\textwidth]{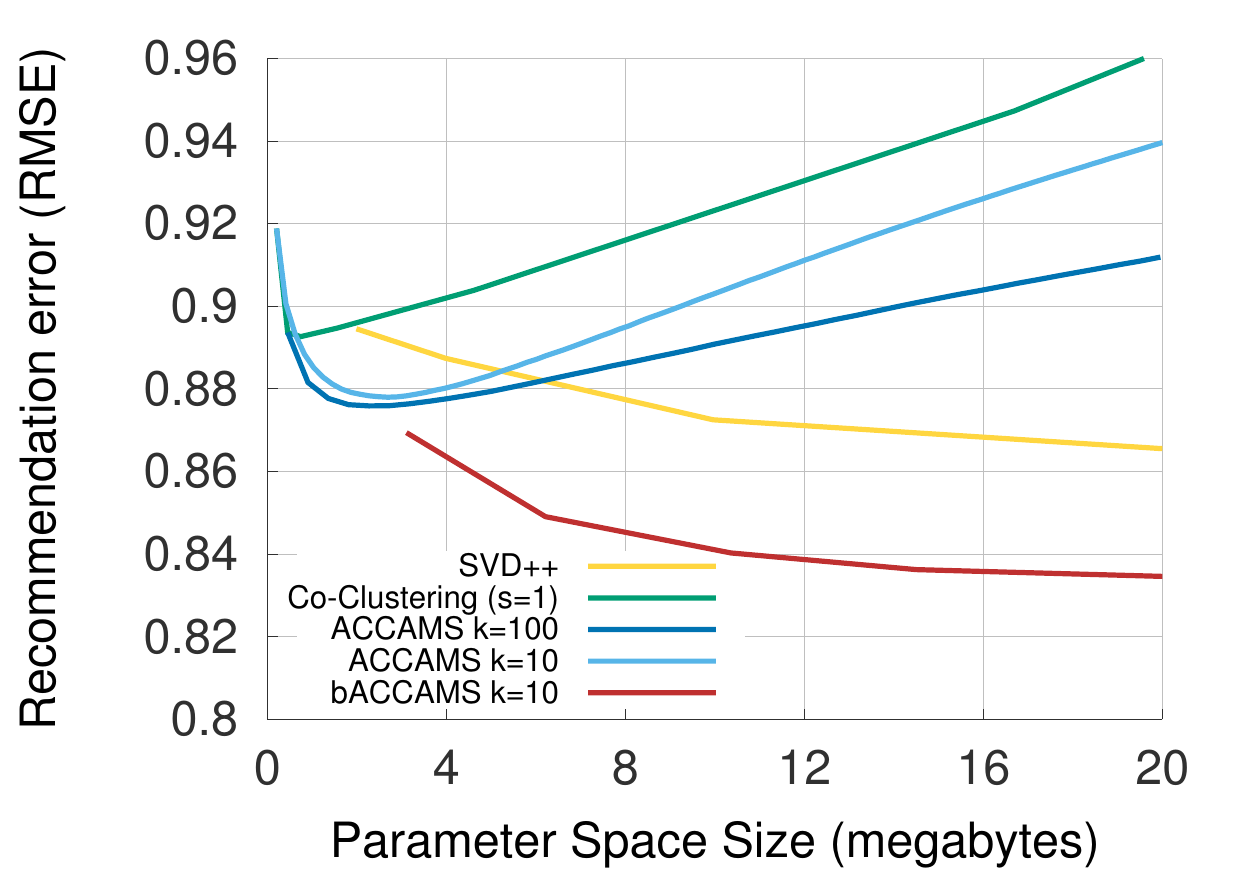} &
  \includegraphics[width=0.37\textwidth]{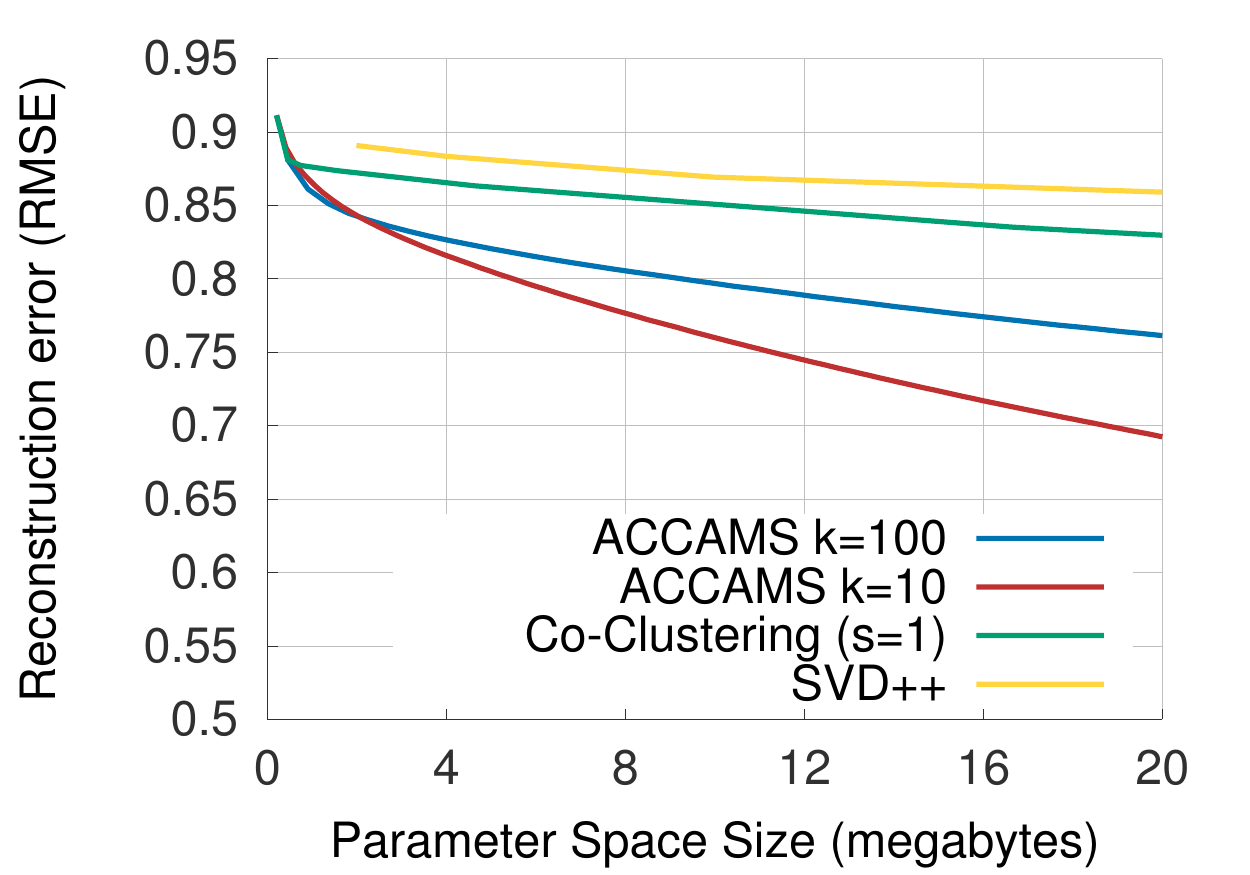} \\
  (a) Netflix Prediction & (b) Netflix Training Error \\
  \includegraphics[width=0.37\textwidth]{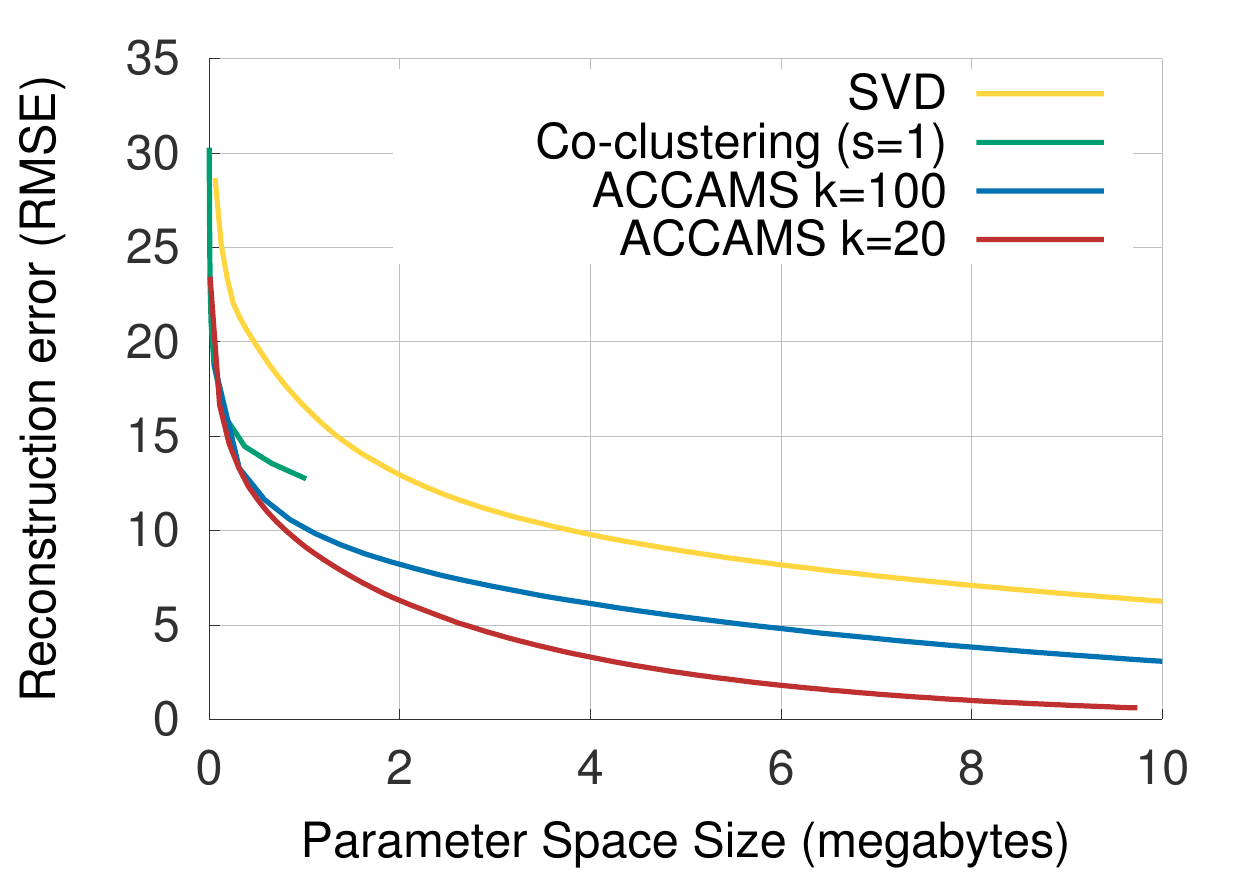} &
  \includegraphics[width=0.37\textwidth]{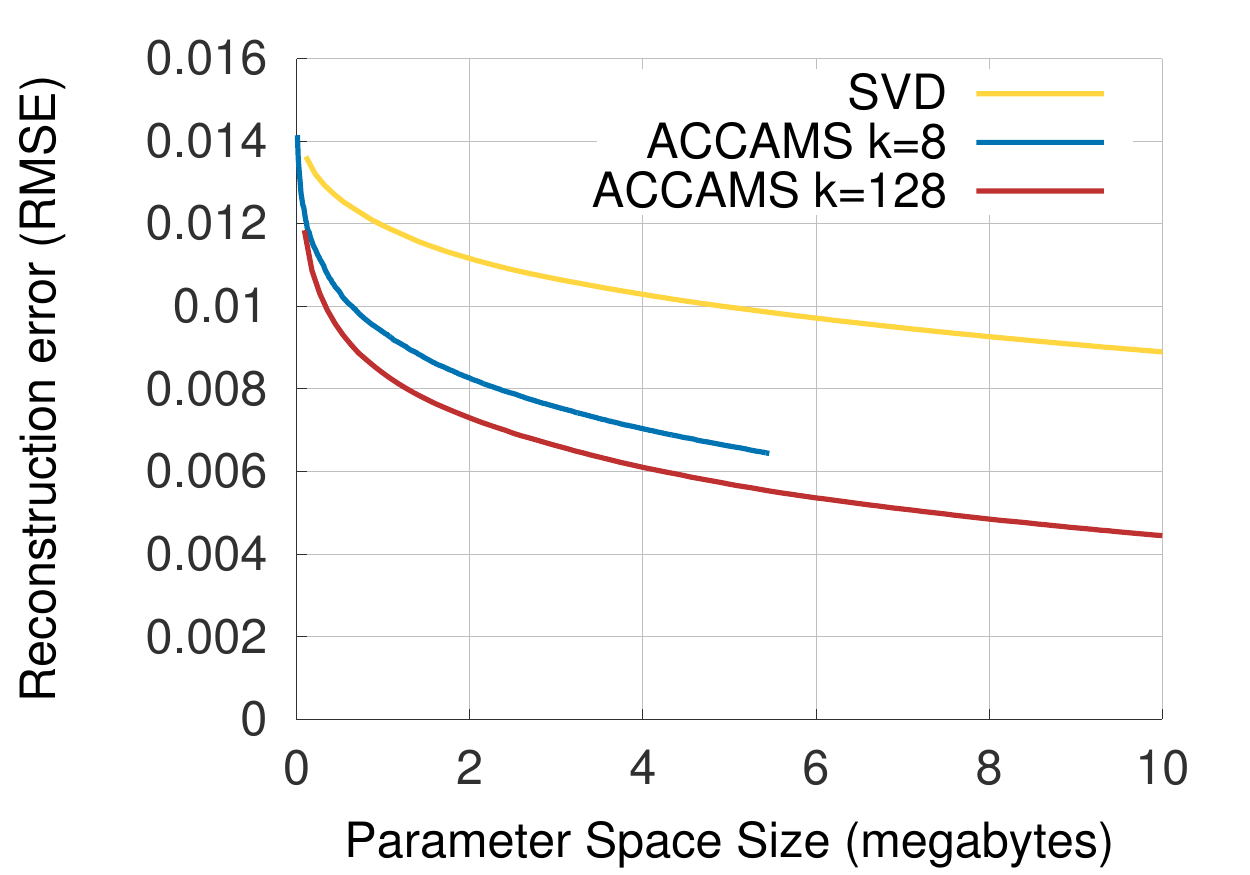}\\
  (c) Faces & (d) AS Graph 
 \end{tabular}
  \caption{
	On images, ratings, and binary graphs,
    \method approximates the matrix more efficiently than SVD, SVD++,
    or classic co-clustering.}
\label{fig:matrix_approx}
\end{figure}

\subsection{Matrix Approximation} 
\label{sub:Matrix Approximation}

In addition to matrix completion, it is valuable to be able to approximate
matrices well, especially for dimensionality reduction tasks.
To test the ability of \method to model matrix data we analyze both how well
our model fits the training data from the Netflix tests above as well as on
image data from the CMU Faces dataset and a binary matrix from the AS peering
graph. 
(Note, for Netflix we now use the training data from one split of the dataset.) 
For each of these of datasets we compare to the SVD (or SVD++ to handle
missing values).  We also use our
algorithm to perform classic co-clustering by setting $s=1$ and varying $k$.

As can be seen in Figure \ref{fig:matrix_approx}(b-d), \method models the matrices
from all three domains much more compactly than SVD (or SVD++ in the case of
the Netflix matrix, which contains missing values).  In particular, we observe
on the CMU Faces matrix that \method uses in some cases under $\frac{1}{4}$ of
the bits as SVD for the same quality matrix approximation.  Additionally, we
observe that using a linear combination of stencils is more efficient to
approximate the matrices than performing classic co-clustering where we have
just one stencil.  
Ultimately, although the method was not designed specifically for image or
network data, we observe that our method is effective for succinctly modeling
the data.

\subsection{Interpretability} 
\label{sub:Interpretability}

\begin{figure}[tb]
\centering
\includegraphics[width=\textwidth]{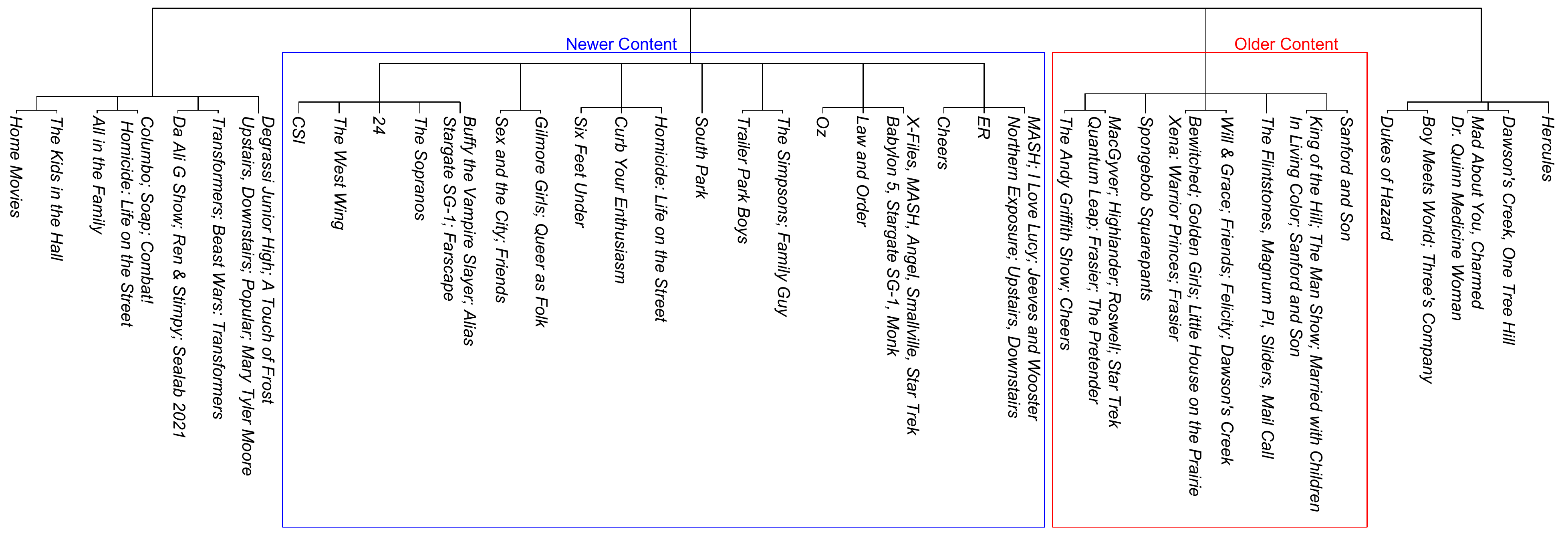} 
\vspace{0mm}
\caption{Hierarchy of TV Shows on Netflix based on the first three stencils generated by \method.}
\label{fig:hierarchy}
\vspace{2mm}
  \tiny
	\centering
	\begin{tabular}{l|l|l|l}
	{\bf 2001: A Space Odyssey}	& {\bf Sex and the City: Season 1}				& {\bf Seinfeld: Seasons 1 \& 2}  	  	& 	    {\bf Mean Girls}   			  			\\ \hline
	Taxi Driver				   & Sex and the City: Season 2					& Seinfeld: Season 3				    &         Clueless						        \\
	Chinatown				   	& Sex and the City: Season 3					& Seinfeld: Season 4				    &         13 Going on 30				        \\
	Citizen Kane			   	& Sex and the City: Season 4					& Curb Your Enthusiasm: Season 1	    &         Best in Show				        \\
	Dr. Strangelove			   & Sex and the City: Season 5					& Curb Your Enthusiasm: Season 2	    &         Particles of Truth				        \\
	A Clockwork Orange		   & Sex and the City: Season 6.1					& Curb Your Enthusiasm: Season 3	    &     Charlie's Angels: Full Throttle               \\
	THX 1138: Special Edition  	& Sex and the City: Season 6.2					& Arrested Development: Season 1	    &       Amelie						          \\
	Apocalypse Now Redux	   	& Hercules:  Season 3							& Newsradio: Seasons 1 and 2		    &         Me Myself I					    \\
	The Graduate			   	& Will \& Grace: Season 1						& The Kids in the Hall: Season 1	    &         Bring it On					        \\
	Blade Runner			   	& Beverly Hills 90210: Pilot					& The Simpsons: Treehouse of Horror     &         Chaos						         \\
	The Deer Hunter			   & The O.C.: Season 1							& Spin City: Michael J. Fox			    &         Kissing Jessica Stein			         \\
	Deliverance				   & Divine Madness
                                                                                                        &
                                                                                                          Curb
                                                                                                          Your
                                                                                                          Enthusiasm:
                                                                                                          Season
                                                                                                          4
                                                                                                                                                        &         Nine Innings from Ground Zero	          \\ 
          \multicolumn{4}{c}{~} \\
		{\bf Star Wars: Episode V} 						& {\bf The Silence of the Lambs}  				& {\bf Scooby-Doo Where Are You?} 		  	 & {\bf Law \& Order: Season 1}				   			\\ \hline
		Star Wars: Episode IV							& The Sixth Sense								& The Flintstones: Season 2				     & Law \& Order: Season 3				                \\
		Star Wars: Episode VI							& Alien: Collector's Edition					& Classic Cartoon Favorites: Goofy           & Law \& Order: SVU (2)	          \\
		Battlestar Galactica: Season 1					& The Exorcist									& Transformers: Season 1 (1984)              & Law \& Order: Criminal Intent (3)        \\
		Raiders of the Lost Ark							& Schindler's List								& Tom and Jerry: Whiskers Away!              & Law \& Order: Season 2				         \\
		Star Wars: Clone Wars: Vol. 1					& The Godfather									& Boy Meets World: Season 1                  & MASH: Season 8						       \\
		Gladiator: Extended Edition						& Seven											& The Flintstones: Season 3                  & ER: Season 1							     \\
		Star Wars Trilogy: Bonus Material				& Colors										& Scooby-Doo's Greatest Mysteries            & MASH: Season 7						      \\
		LOTR: The Fellowship of the Ring				& The Godfather, Part II 						& Care Bears: Kingdom of Caring              & Rikki-Tikki-Tavi								\\
		LOTR: The Two Towers				& GoodFellas: Special Edition								& Aloha Scooby-Doo!							 & The X-Files: Season 6					   			\\
		Indiana Jones Trilogy: Bonus Material			& Platoon										& Scooby-Doo: Legend of the Vampire			 & The X-Files: Season 7					                    \\
		LOTR: The Return of the King					& Full Metal Jacket								& Rugrats: Decade in Diapers 	             & ER: Season 3							         
	\end{tabular}
	\caption{For a given movie or TV show on Netflix, we can use the cluster assignments to find related content.}
	\label{tab:netflix_similar}
\end{figure}

In any model the structure of factors makes assumptions about the
form of user preferences and decision making.  The fact that our model
achieves high-quality matrix completion with a smaller parameter space suggests
that our modeling assumptions better match how people make decisions.
One side effect of our model being both compact and conceptually
simple is that we can understand our learned parameters.

To test the model's interpretability we use \method to model the Netflix data
with $s=20$ stencils and $k^2=100$ clusters (a model of similar size to a
rank-3 matrix factorization).   
Here we look at two ways to interpret the results.

First we view the cluster assignments in stencils as inducing a hierarchy on
the movies.  That is, movies are split in the first level based on their
cluster assignments in the first stencil.  At the second level, we split movies
based on their cluster assignments in the second stencil, etc.
In Figure \ref{fig:hierarchy} we observe the hierarchy of TV shows induced by
the first three stencils learned by \method (we only include shows
where there is more than one season of that show in the leaf and we
pruned small partitions due to space restrictions).

As can be seen in the hierarchy, there are branches which clearly cluster together
shows more focused on male audiences, female audiences, or children.
However, beyond a first brush at the leaf nodes, we can notice some larger
structural differences.  For example, looking at the two large branches coming
from the root, we observe that the left branch generally contains more recent
TV shows from the late 1990s to the present, while the right branch generally
contains older shows ranging from the 1960s to the mid 1990s.  This can be most
starkly noticed by ``Friends,'' which shows up in both branches; Seasons 1 to
4 of ``Friends'' from 1994-1997 fall in the older branch, while Seasons 5 to 9
from 1998-2002 fall in the newer branch.
Of course the algorithm does not know the dates the shows were released, but
our model learns these general concepts just based on the ratings.
From this it is clear the stencils can be useful for breaking down content in a
meaningful structured way, something that is not possible under classic
factorization approaches.

While the hierarchy demonstrates that our stencils are learning meaningful
latent factors, it may be difficult to always understand individual clusters.
Rather, to use knowledge from {\em all} of the stencils, we can look to  the use case of
``Users who watched $X$ also liked $Y$,'' and ask given a movie or TV show to
search, can we find other similar items?  We do this by comparing the set of
cluster assignments from the given movie to the set of cluster assignments of
other items. We measure similarity between two movies using the
Hamming distance between cluster assignments.

As can be seen in Table \ref{tab:netflix_similar}, we find the combination of
clusters for different movies and TV shows can be used to easily find similar
content.  While we see some obvious cases where the method succeeds, e.g. ``Sex
and the City'' returns six more seasons of ``Sex and the City,'' we also notice
the method takes into account more subtle similarities of movies beyond genre.
For example, while the first season of ``Seinfeld'' returns the subsequent seasons
of ``Seinfeld,'' it is followed by three seasons of ``Curb Your Enthusiasm,''
another comedy show by the same writer Larry David.  Similarly, searching for Stanley
Kubrick's ``2001: A Space Odyssey'' returns other Stanley Kubrick movies, as
well as other critically acclaimed films from that era, particularly
thematically similar science fiction movies.  
Searching for
``Scooby-Doo'' returns topically similar children's shows, specifically from
the mid to late 1900's.  From this we get a sense that \method does
not just find similarity in genre but also more subtle similarities.

\begin{figure}[tb]
  \centering
\begin{tabular}{ccc}
  Original & Stencil 1 & Stencil 2 \\
  \includegraphics[width=0.22\columnwidth]{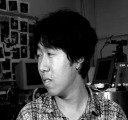} &
  \includegraphics[width=0.22\columnwidth]{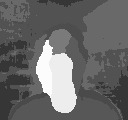}  &
  \includegraphics[width=0.22\columnwidth]{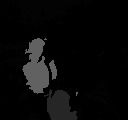}  \\
  \includegraphics[width=0.22\columnwidth]{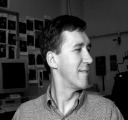}  &
  \includegraphics[width=0.22\columnwidth]{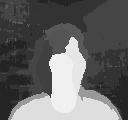} &
  \includegraphics[width=0.22\columnwidth]{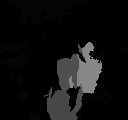}  
\end{tabular}
\caption{Examples of original images and the first two stencils. 
  The decomposition is very similar to that of eigenfaces \cite{TurPen91},
  albeit much more concise in its nature.}
\label{fig:faces}
\end{figure}


\subsection{Properties of ACCAMS} 
\label{sub:properties}

\begin{figure}[tb]
\centering
\includegraphics[width=0.40\columnwidth]{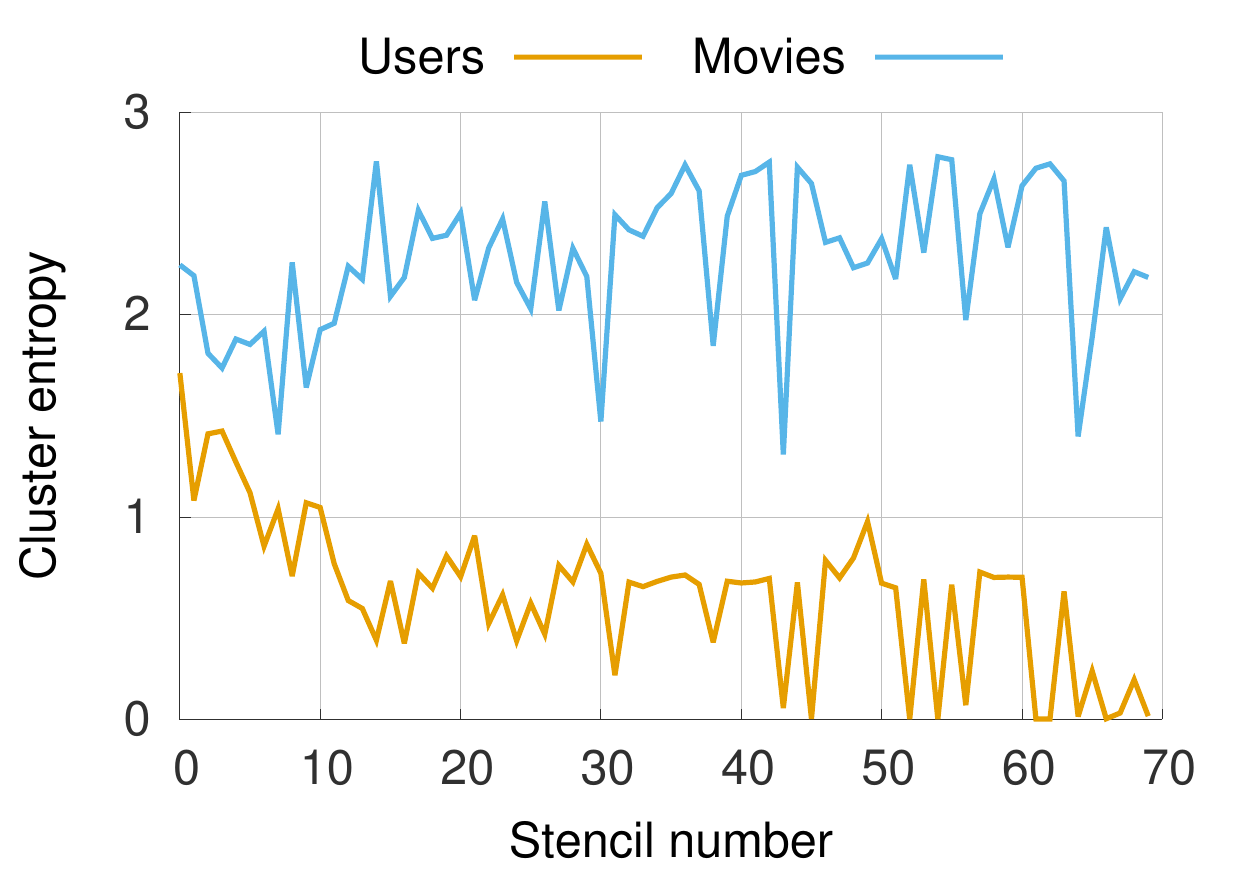}
\includegraphics[width=0.40\columnwidth]{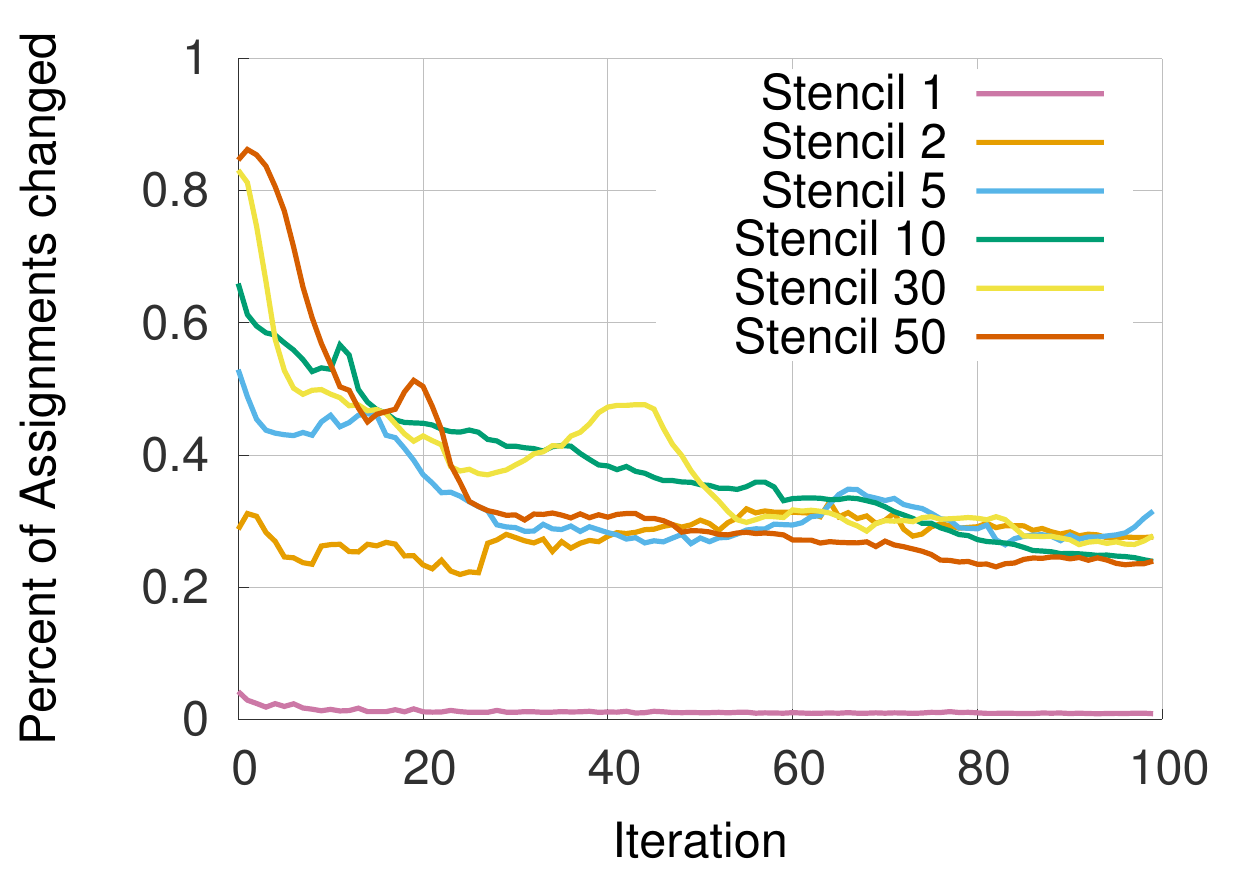} 
\caption{Left: Entropy in cluster assignments for users and
  movies. Right: Stability of the assignments in the sampler.}
\label{fig:analysis}
\end{figure}

Aside from \method's success across matrix completion and approximation, it is
valuable to understand how our method is working, particularly because of how
different it is from previous models.
First, because \method uses backfitting, we expect that the first
stencil captures the largest features, the second captures secondary ones,
etc.  This idea is backed up by our theoretical results in Section
\ref{sec:entropy}, and we observe that this is working experimentally by the
drop off in RMSE for our matrix approximation results in Figure
\ref{fig:matrix_approx}.
We can visually observe this in the image approximation of the
CMU Faces.  As can be seen in Figure \ref{fig:faces}, the first stencil
captures general structures of the room and heads, and the second starts to
fill in more fine grained details of the face.  

The Bayesian model, \bmethod, backfits in the first iteration of the
sampler but ultimately resamples each stencil many times thus
loosening these properties.  In Figure \ref{fig:analysis}, we observe
how the distribution of users and movies across clusters changes over
iterations and number of stencils, based on our run of \bmethod with
$s=70$ stencils and a maximum of $k=10$ clusters per stencil.  As we
see in the plot of entropy, movies, across all 70 stencils, are well distributed
across the 10 possible clusters. 
Users, however, are well distributed in the early stencils but then are only
spread across a few clusters in later stencils.
In addition, we
notice that while the earlier clusters are stable,
later stencils are much less stable with a high percentage of cluster
assignments changing.  Both of these properties follow
from the fact that most users rate very few movies.  For
most users only a few clusters are necessary to capture their
observed preferences. Movies, however, typically have more ratings and 
more latent information to infer. Thus through all 70
stencils we learn useful clusterings, and our
prediction accuracy improves through $s=125$ stencils.

\section{Discussion}
\label{sec:discussion}

Here we formulated a model of additive co-clustering.  
We presented both a $k$-means style algorithm, \method, as well as a generative
Bayesian non-parametric model with a collapsed Gibbs sampler,
\bmethod; we obtained theoretical guarantees for matrix approximation
through additive co-clustering; and we showed that our method is concise and
accurate on a diverse range of datasets, including achieving the best published
accuracy on Netflix.

Given the novelty and initial success of the
method, we believe that domain-specific variants of \method, such as for
community detection and topic modeling, can and will lead to new models and
improved results.  In addition, given the modularity of our framework, it is
easy to incorporate side information, such as explicit genre and actor data, in
modeling rating data that should lead to improved accuracy and
interpretability.

{\small
\myparagraph{Acknowledgements}
We would like to thank Christos Faloutsos for his valuable feedback 
throughout the preparation of this paper.
This research was supported by funds from Google,
a Facebook Fellowship, a
National Science Foundation Graduate Research Fellowship (Grant No.
\textsc{DGE-1252522}), and
the National Science Foundation
under Grant No. CNS-1314632 and IIS-1408924. 
Any opinions, findings, and conclusions or recommendations expressed in this
material are those of the author(s) and do not necessarily reflect the views
of the National Science Foundation, or other funding parties. 
}

\clearpage
\renewcommand{\refname}{\normalfont\selectfont\normalsize\bfseries{References}}
\bibliographystyle{plain}


\end{document}